\documentclass[a4paper,10pt,conference]{ieeeconf}
\IEEEoverridecommandlockouts
\usepackage{amsmath,amssymb,bm}

\usepackage{amsthm}
\usepackage{graphicx}
\usepackage{booktabs}
\usepackage{algorithm}
\usepackage{algorithmic}
\usepackage{cite}
\usepackage{url}

\newtheorem{proposition}{Proposition}
\theoremstyle{definition}
\newtheorem{definition}{Definition}

\newcommand{\Real}{\mathbb{R}}
\newcommand{\U}{\mathcal{U}}
\newcommand{\Lset}{\mathcal{L}}
\newcommand{\simplex}{\Delta}

\graphicspath{{assets/}}

\title{\LARGE \bf
Joint-Feasibility Margin Feedback for Bounded-Input Second-Order Control Barrier Functions
}

\author{Siyuan Yang$^{1}$ and Shaoming He$^{1,2}$%
\thanks{This work has been submitted to the IEEE for possible publication. Copyright may be transferred without notice, after which this version may no longer be accessible.}%
\thanks{$^{1}$Siyuan Yang and Shaoming He are with the School of Aerospace
Engineering, Beijing Institute of Technology, Beijing 100081, China.}%
\thanks{$^{2}$Shaoming He is the corresponding author
(e-mail: shaoming.he@bit.edu.cn).}%
}

\begin{document}
\bstctlcite{ROBIO2026BSTControl}

\maketitle
\thispagestyle{empty}
\pagestyle{empty}

\begin{abstract}
Multiple second-order control barrier function constraints can exhaust their common control authority under bounded acceleration, even when the current state remains feasible.
We define a signed local joint-feasibility margin whose sign exactly identifies whether the current barrier half-spaces intersect the admissible input set.
For unit-scaled distance constraints, it has acceleration units and a support-function dual that identifies weighted constraint conflicts.
We predict the margin one sample ahead for nine task-directed candidates, selecting either the least-deviating candidate that attains a desired margin or the candidate with the largest predicted margin before solving the same hard HOCBF quadratic program.
A controlled one-step case shows that, at the state predicted under the hard-QP command, two constraints remain individually feasible but become jointly infeasible, whereas the margin-feedback controller retains a positive next margin.
Separately, across 20 paired initial conditions, the margin-feedback controller retains the hard-QP command in a low-pressure task and acts at positive margin in a critical task before the hard-QP problem becomes infeasible; its trajectories remain feasible at that time and complete the task.
Source code is publicly available.\footnote{\url{https://github.com/xirhxq/cbf/tree/robio2026-preventive-regulation}}
\end{abstract}

\begin{keywords}
Control barrier functions, input constraints, joint feasibility, predictive safety filtering, multi-UAV systems.
\end{keywords}

\section{Introduction}
\label{sec:intro}

Control barrier functions (CBFs) turn state constraints into inequalities on the control input and are widely used as real-time safety filters \cite{Ames2017CBFQP,Ames2019CBFQP}.
For systems controlled through acceleration, collision clearance, communication range, and other distance requirements have relative degree two and can be enforced with high-order CBFs (HOCBFs) \cite{NguyenSreenathECBF,Xiao2019HOCBF,Xiao2021HOCBF}.
Each requirement is simple in isolation, but all of its HOCBF inequalities share the same bounded acceleration.
Their common admissible set can therefore shrink rapidly or disappear as relative speed and braking demand increase.

A conventional hard HOCBF quadratic program (QP) reports this conflict only after no input satisfies all inequalities.
The task controller may still be driving toward the boundary while the current QP remains feasible.
This motivates two questions: how much common control authority remains at the current state, and can that quantity modify the task command before it changes sign?

The first quantity is a signed \emph{joint-feasibility margin}: the largest common slack of each UAV's local HOCBF constraints over its admissible acceleration set.
Its sign gives an exact pointwise feasibility test, while its dual assigns a separating weight to conflicting constraints.
To act before this margin changes sign, a one-step feedback law evaluates nine commands between the task-directed input and a maximum-margin input, then selects either the least-deviating candidate that reaches a desired margin or the candidate with the largest predicted margin.
The selected command remains the nominal input of the same hard HOCBF-QP.
In the evaluation, Proposed denotes this margin-feedback controller, while Baseline denotes the unchanged hard HOCBF-QP.

The contributions are twofold.
First, we formulate a local joint-feasibility margin in acceleration units for normalized second-order distance constraints and give its exact feasibility and dual interpretations.
Second, we develop a deterministic one-step margin-feedback controller that balances task tracking and predicted common control authority without a finite-horizon optimization.

We first isolate the one-step mechanism in a controlled case with one UAV and two fixed references.
We then evaluate the method in a post-detection maritime response task with four UAVs and two fixed references.
In 20 paired initial conditions, the Proposed and Baseline controllers behave identically in a low-pressure task.
At a critical but statically reachable distance, Proposed first acts while its current margin is positive; Baseline subsequently reaches local infeasibility, whereas Proposed retains positive margin at the same time and completes reach-and-dwell.
These trajectories show preventive margin regulation in the matched task.

\section{Related Work}
\label{sec:related}

CBF-QPs combine a nominal controller with affine safety constraints and recover the nominal command whenever it is already admissible \cite{Ames2017CBFQP,Ames2019CBFQP}.
Exponential CBFs and HOCBFs extend this construction to constraints of higher relative degree \cite{NguyenSreenathECBF,Xiao2019HOCBF,Xiao2021HOCBF}.
Input-constrained variants instead construct subsets or barrier conditions for which a bounded controller exists \cite{Agrawal2021ICCBF,Breeden2021HIRDCBF,Breeden2022Robust}.
These methods provide safe-controller synthesis, but the presence of several HOCBF constraints still raises a separate compatibility question.

Several studies analyze that question directly.
Tan and Dimarogonas use a pointwise max--min compatibility score with grid refinement to verify multiple CBFs over a state domain \cite{Tan2022Compatibility}.
Isaly et al. study feasibility and continuity of controllers defined by vector-valued barrier functions \cite{Isaly2024MultipleCBFFeasibility}.
Other methods construct compatible or consolidated barrier families, or adapt HOCBF parameters to preserve feasibility \cite{Breeden2023Composition,Black2023Consolidated,Zhang2023Feasibility}.
In contrast to state-domain verification, we compute at each sample an online signed bounded-input margin for the current finite set of second-order distance constraints and use its one-step prediction to select the task command.

Predictive safety filters solve a finite-horizon problem and use a backup trajectory toward a safe terminal set \cite{Wabersich2021PredictiveSafety}.
Predictive CBFs instead encode the future safety of a propagated nominal trajectory into a barrier condition \cite{Breeden2022PredictiveCBF}.
Both address limitations of purely pointwise filtering, but with different predictive objects and guarantees.
Here a finite set of accelerations is propagated for one sampling period, and predicted joint feasibility selects the nominal command supplied to the same hard HOCBF-QP.

\section{Problem Formulation}
\label{sec:problem}

Let $\mathcal N=\{1,\ldots,N\}$ index $N$ UAVs moving in the plane,
\begin{equation}
    \dot p_i=v_i,\qquad \dot v_i=u_i,
    \label{eq:dynamics}
\end{equation}
where $p_i,v_i,u_i\in\Real^2$ are the position, velocity, and acceleration of UAV $i\in\mathcal N$.
Throughout, $\|\cdot\|_2$ denotes the Euclidean norm.
The acceleration belongs to a nonempty compact convex set
$\U_i\subset\Real^2$.
A task controller supplies a bounded nominal acceleration $u_i^{\rm task}$.
At samples $t_k=kT_s$, $k=0,1,\ldots$, with sampling period $T_s>0$, the local model uses
$\bar u_j[k]=u_j[k-1]$, the mobile neighbor's actual input held over the preceding control interval, with specified initial values $u_j[-1]$.
Let $x[k]=(p_1[k],v_1[k],\ldots,p_N[k],v_N[k])$ be the sampled mobile-UAV joint state and define $\bar u_{-i}[k]=(\bar u_j[k])_{j\in\mathcal N\setminus\{i\}}$, where $-i$ denotes the other mobile UAVs.
At a fixed sample, we write $x=x[k]$.
In pairwise expressions, $j$ may denote a mobile neighbor or a fixed reference $b_r$ at $p_{b_r}\in\Real^2$, $r\in\{0,1\}$; the latter has $v_{b_r}=\bar u_{b_r}[k]=0$.

Let $d_{ij}=\|p_i-p_j\|_2$.
The constants $d_{\min}>0$ and $d_{\max}>d_{\min}$ denote the minimum permitted separation and maximum permitted communication distance, respectively.
Collision clearance and communication range are represented by
\begin{equation}
 h_{\rm safe}^{ij}=d_{ij}-d_{\min},\qquad
 h_{\rm com}^{ij}=d_{\max}-d_{ij}.
 \label{eq:distance-barriers}
\end{equation}
For either barrier $h$, positive gains $\lambda_1,\lambda_2$ with units $\mathrm{s^{-1}}$ give the second-order condition
\begin{equation}
 \ddot h+(\lambda_1+\lambda_2)\dot h+\lambda_1\lambda_2 h\geq0.
 \label{eq:hocbf}
\end{equation}
The corresponding high-order safe set also requires
$h\geq0$ and $\dot h+\lambda_1 h\geq0$ initially.

The controller domain contains only constrained pairs with $d_{ij}>0$.
Define $n_{ij}=(p_i-p_j)/d_{ij}$ and $v_{ij}=v_i-v_j$ on this domain.
Because
\begin{equation}
 \ddot d_{ij}=n_{ij}^{\top}(u_i-u_j)
 +\frac{\|v_{ij}\|_2^2-(n_{ij}^{\top}v_{ij})^2}{d_{ij}},
 \label{eq:distance-acceleration}
\end{equation}
each HOCBF condition is affine in $u_i$ after the neighbor acceleration used for the current update is fixed.
Let the transverse kinematic term
$q_{ij}=(\|v_{ij}\|_2^2-(n_{ij}^{\top}v_{ij})^2)/d_{ij}$, measured in $\mathrm{m/s^2}$.
The collision and communication conditions respectively become
\begin{align}
 n_{ij}^{\top}u_i &\geq n_{ij}^{\top}\bar u_j[k]-q_{ij} \notag\\
 &\quad -(\lambda_1+\lambda_2)n_{ij}^{\top}v_{ij}
 -\lambda_1\lambda_2(d_{ij}-d_{\min}), \label{eq:safe-constraint}\\
 -n_{ij}^{\top}u_i &\geq -n_{ij}^{\top}\bar u_j[k]+q_{ij} \notag\\
 &\quad +(\lambda_1+\lambda_2)n_{ij}^{\top}v_{ij}
 -\lambda_1\lambda_2(d_{\max}-d_{ij}). \label{eq:comm-constraint}
\end{align}
Let $\Lset_i$ denote all current constraints of UAV $i$, written as
\begin{equation}
 c_{i\ell}(x)^{\top}u_i\geq
 \beta_{i\ell}(x;\bar u_{-i}[k]),
 \qquad \ell\in\Lset_i,
 \label{eq:local-constraint}
\end{equation}
where these constraints retain their identities while their coefficients depend on the state.
Equations~\eqref{eq:safe-constraint} and \eqref{eq:comm-constraint} give
$c_{i\ell}(x)=n_{ij}$ for collision clearance and $c_{i\ell}(x)=-n_{ij}$ for communication range, with $\|c_{i\ell}(x)\|_2=1$ and $\beta_{i\ell}(x;\bar u_{-i}[k])$ measured in $\mathrm{m/s^2}$.

Whenever the local constraints are feasible, the hard-QP command is the unique solution of
\begin{equation}
 \begin{aligned}
 u_i^0={}&\arg\min_{u\in\U_i}\frac12\|u-u_i^{\rm task}\|_2^2\\
 &\mathrm{s.t.}\quad c_{i\ell}(x)^{\top}u\geq
 \beta_{i\ell}(x;\bar u_{-i}[k]),\quad \ell\in\Lset_i.
 \end{aligned}
 \label{eq:baseline-qp}
\end{equation}
Our objective is to change the nominal supplied to this hard HOCBF-QP when a one-step forecast indicates loss of common control authority, while retaining \eqref{eq:local-constraint} as hard constraints.

\section{Joint-Feasibility Margin}
\label{sec:margin}

\begin{definition}[Local joint-feasibility margin]
For a nonempty local constraint set $\Lset_i$, define
\begin{equation}
 \gamma_i^\star(x;\bar u_{-i}[k])=
 \max_{u\in\U_i}\ \min_{\ell\in\Lset_i}
 \big(c_{i\ell}(x)^{\top}u
 -\beta_{i\ell}(x;\bar u_{-i}[k])\big).
 \label{eq:margin}
\end{equation}
\end{definition}

With unit-normal scaling, $\gamma_i^\star$ has acceleration units; its sign characterizes feasibility for the preceding neighbor accelerations used at that sample, while its magnitude depends on the chosen constraint scaling.

Let $\bm{1}\in\Real^{|\Lset_i|}$ be the all-ones vector,
$\simplex_i=\{\mu\in\Real^{|\Lset_i|}:\mu\geq0,\bm{1}^{\top}\mu=1\}$,
and $\sigma_{\U_i}(z)=\max_{u\in\U_i}z^{\top}u$ for $z\in\Real^2$ be the support function of the input set.
All sums over $\ell$ below range over $\Lset_i$.

\begin{proposition}[Feasibility and dual interpretation]
The local constraints \eqref{eq:local-constraint} admit an input in $\U_i$ if and only if $\gamma_i^\star\geq0$.
Moreover,
\begin{equation}
 \gamma_i^\star=
 \min_{\mu\in\simplex_i}
 \left[
 \sigma_{\U_i}\!\left(\sum_{\ell}\mu_\ell c_{i\ell}(x)\right)
 -\sum_{\ell}\mu_\ell\beta_{i\ell}(x;\bar u_{-i}[k])
 \right].
 \label{eq:margin-dual}
\end{equation}
Thus $\gamma_i^\star<0$ exactly when a convex combination of the constraint requirements exceeds the maximum support available over $\U_i$.
\end{proposition}

\begin{proof}
If \eqref{eq:local-constraint} is feasible, its feasible input makes the inner minimum in \eqref{eq:margin} nonnegative.
Conversely, an optimizer with $\gamma_i^\star\geq0$ satisfies every constraint.
Writing the inner minimum as a minimization over $\simplex_i$ gives a bilinear objective on compact convex sets.
The minimax theorem exchanges the two optimizations, and maximizing over $u\in\U_i$ yields the support function in \eqref{eq:margin-dual}.
\end{proof}

For the componentwise box $\U_i=[-a_{\max},a_{\max}]^2$ used in our simulations, where $a_{\max}>0$ is the acceleration bound,
$\sigma_{\U_i}(z)=a_{\max}\|z\|_1$, with $\|\cdot\|_1$ denoting the one-norm.
The minimizing $\mu$ identifies the limiting weighted constraint combination without reference to the quadratic cost in \eqref{eq:baseline-qp}.

\section{One-Step Margin Feedback}
\label{sec:feedback}

The margin-feedback controller uses the feasible task-directed command $u_i^0$ from \eqref{eq:baseline-qp} and a maximum-margin endpoint
\begin{equation}
 u_i^{\max}=\min_{\prec_{\rm lex}}
 \arg\max_{u\in\U_i}
 \min_{\ell\in\Lset_i}
 \big(c_{i\ell}(x)^{\top}u
 -\beta_{i\ell}(x;\bar u_{-i}[k])\big),
 \label{eq:max-margin-input}
\end{equation}
where $\prec_{\rm lex}$ is a fixed lexicographic order on acceleration components.
It evaluates the nine uniformly interpolated inputs
\begin{equation}
 \begin{aligned}
 u_{i,m}&=(1-\alpha_m)u_i^0+\alpha_m u_i^{\max},\\
 \alpha_m&=m/8,\qquad m=0,\ldots,8.
 \end{aligned}
 \label{eq:candidates}
\end{equation}
When $\gamma_i^\star\geq0$, convexity makes every candidate feasible for the current local constraints.

For candidate $m$ of UAV $i$, define the synchronous task-reference surrogate by
$\tilde u_i=u_{i,m}$ and $\tilde u_j=u_j^0$ for $j\neq i$.
Let the superscript ${}^+$ denote a one-step prediction at $t_{k+1}$.
All mobile states are propagated with zero-order hold,
\begin{equation}
 p_j^+=p_j+T_s v_j+\tfrac12T_s^2\tilde u_j,
 \qquad v_j^+=v_j+T_s\tilde u_j.
 \label{eq:zoh-prediction}
\end{equation}
Fixed references remain stationary, $p_{b_r}^+=p_{b_r}$ and
$v_{b_r}=\bar u_{b_r}[k]=\tilde u_{b_r}=0$.
The state-dependent coefficients are then rebuilt for all corresponding constraints using $\tilde u_j$ as the neighbor acceleration.
Writing $\tilde u_{-i}=\{\tilde u_j:j\in\mathcal N\setminus\{i\}\}$ and
$x_{i,m}^+=(p_1^+,v_1^+,\ldots,p_N^+,v_N^+)$ for the resulting predicted mobile-UAV joint state, the score is
$g_{i,m}=\gamma_i^\star(x_{i,m}^+;\tilde u_{-i})$,
where the other UAVs use their task-directed safe commands $u_j^0$ in this surrogate.
The associated task-deviation score is
$J_{i,m}=\|u_{i,m}-u_i^0\|_2$.
Both $g_{i,m}$ and $J_{i,m}$ have units of $\mathrm{m/s^2}$.

The threshold $\tau>0$, measured in $\mathrm{m/s^2}$, ranks predicted candidates; zero remains the hard-feasibility boundary.
Write $u_{i,m}=(u_{i,m,x},u_{i,m,y})$.
Let $\mathcal M_i=\{0,\ldots,8\}$ and
$\mathcal M_i^\tau=\{m\in\mathcal M_i:g_{i,m}\geq\tau\}$.
If $\mathcal M_i^\tau\neq\emptyset$, the controller minimizes
$(J_{i,m},-g_{i,m},u_{i,m,x},u_{i,m,y},m)$ lexicographically over $\mathcal M_i^\tau$.
Otherwise it minimizes
$(-g_{i,m},J_{i,m},u_{i,m,x},u_{i,m,y},m)$ over $\mathcal M_i$.
The first case is the least-deviation branch; the second is the maximum-margin fallback.
Let $m_i^\star$ denote the minimizing index under the applicable rule and
$u_i^{\rm sel}=u_{i,m_i^\star}$.
This selected candidate replaces $u_i^{\rm task}$ in the hard HOCBF-QP \eqref{eq:baseline-qp}.

Algorithm~\ref{alg:feedback} summarizes the controller.

\begin{algorithm}[t]
\caption{One-step joint-feasibility-margin feedback}
\label{alg:feedback}
\footnotesize
\begin{algorithmic}[1]
\REQUIRE Current states, $\{\bar u_j[k]\}$, all current constraints, task inputs, $T_s$, and $\tau$
\ENSURE Hard HOCBF-QP commands $\{u_i\}$, or \textsc{Failure}
\FOR{each UAV $i=1,\ldots,N$}
  \STATE Solve \eqref{eq:margin} and choose $u_i^{\max}$ by \eqref{eq:max-margin-input}
  \IF{$\gamma_i^\star<0$}
    \STATE \textbf{return} \textsc{Failure} \COMMENT{no input satisfies all current constraints}
  \ENDIF
  \STATE Compute the unique task-directed hard HOCBF-QP command $u_i^0$
\ENDFOR
\FOR{each UAV $i=1,\ldots,N$}
  \FOR{$m=0,\ldots,8$}
    \STATE Form $u_{i,m}$ by \eqref{eq:candidates}
    \STATE Propagate one step with $u_{i,m}$ and $\{u_j^0\}_{j\neq i}$
    \STATE Evaluate predicted margin $g_{i,m}$ and deviation $J_{i,m}$
  \ENDFOR
  \IF{$\{m:g_{i,m}\geq\tau\}\neq\emptyset$}
    \STATE Select the least-deviation candidate using the stated tie order
  \ELSE
    \STATE Select the maximum-margin candidate using the stated tie order
  \ENDIF
  \STATE Solve the current hard HOCBF-QP centered at $u_i^{\rm sel}$
\ENDFOR
\STATE Apply all UAV commands simultaneously
\end{algorithmic}
\end{algorithm}

This is a one-step selection rule; the hard HOCBF-QP checks feasibility again at the next sample.

\section{Numerical Evaluation}
\label{sec:evaluation}

\subsection{Controlled One-Step Case}

We first isolate the one-step mechanism using one UAV and two fixed-reference communication constraints under the exact ZOH model.
For this single-UAV case, we omit the UAV subscript.
The setup uses $T_s=0.5$ s, input box $[-2,2]^2\,\mathrm{m/s^2}$, $p_{b_0}=(0,-100)$ m and $p_{b_1}=(0,100)$ m, $d_{\max}=500$ m, $\lambda_1=\lambda_2=1\,\mathrm{s^{-1}}$, $\Delta y\in\{-0.50,-0.25,0,0.25,0.50\}$ m, $p[0]=(\sqrt{450^2-100^2},\Delta y)$ m, $v[0]=(18,0)$ m/s at low pressure and $(20.45,0)$ m/s at critical pressure, and $u^{\rm task}=(2,0)\,\mathrm{m/s^2}$.
The controlled case uses $\tau=1\,\mathrm{m/s^2}$, which lies between the central $m=2$ and $m=3$ candidate margins of $0.912$ and $1.462\,\mathrm{m/s^2}$; the multi-UAV case uses $\tau=14\,\mathrm{m/s^2}$, whose absolute value is not directly comparable because the state-dependent constraint offsets differ.
At low pressure, the margin-feedback controller selects $m=0$ and therefore returns the hard-QP command unchanged, with next margin $5.804\,\mathrm{m/s^2}$.
At critical pressure, the current joint margin is $12.027\,\mathrm{m/s^2}$, but the hard-QP command produces a nominal-next joint margin of $-0.189\,\mathrm{m/s^2}$.
At that predicted-next state, both singleton margins remain positive ($0.148$--$0.343\,\mathrm{m/s^2}$), and both dual weights are strictly positive ($0.498$--$0.502$), so both constraints participate in the same joint-margin certificate.
The margin-feedback controller selects $m=3$, or $\alpha=3/8$; all five predicted and actual next margins round to $1.462\,\mathrm{m/s^2}$ and exceed $\tau$.
Figure~\ref{fig:controlled-mechanism} shows how this preventive command restores positive predicted joint margin before the current constraints become infeasible.

\begin{figure*}[t]
\centering
\includegraphics[width=0.98\textwidth]{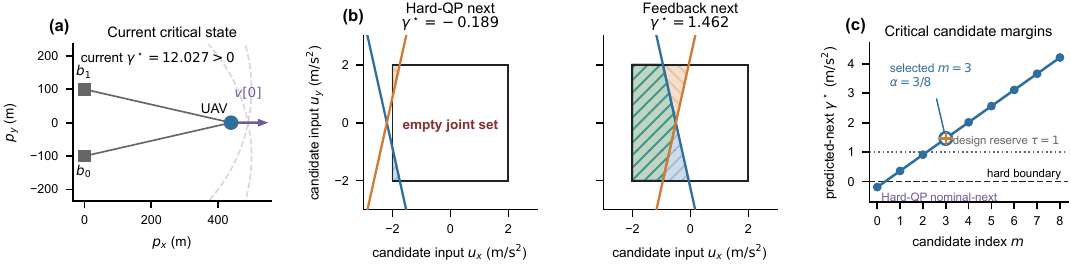}
\caption{Controlled one-step joint-constraint case. (a) The current critical state is jointly feasible. (b) The states predicted under the hard-QP and margin-feedback commands have different next-step feasible sets. (c) Median critical candidate margins across five fixed offsets; the five curves overlap at this scale.}
\label{fig:controlled-mechanism}
\end{figure*}

\subsection{Multi-UAV Task and Setup}

We consider a post-detection response phase in which the horizontal position of a stationary vessel is known.
Two fixed references $b_0$ and $b_1$ support four mobile UAVs.
We adapt the static reference construction in \cite{Yang2024MaritimeSearch} to the response task.
The assigned communication edges are
\begin{multline}
 \{\mathrm{UAV1}\!\to b_0,\mathrm{UAV1}\!\to b_1,
 \mathrm{UAV2}\!\to b_0,\\
 \mathrm{UAV2}\!\to\mathrm{UAV1},
 \mathrm{UAV3}\!\to\mathrm{UAV1},
 \mathrm{UAV3}\!\to\mathrm{UAV2},\\
 \mathrm{UAV4}\!\to\mathrm{UAV2},
 \mathrm{UAV4}\!\to\mathrm{UAV3}\}.
 \label{eq:edge-set}
\end{multline}
The resulting single triangular ladder remains fixed throughout the task.
For each UAV, $\Lset_i$ contains its three mobile-UAV collision constraints and the assigned outgoing communication constraints in \eqref{eq:edge-set}; these identities do not switch, although their coefficients change with the state.
UAV4 completes the task after remaining within 50 m of the vessel position for 30 continuous seconds; leaving the neighborhood resets the dwell timer.

The physical communication limit is 850 m and the minimum UAV separation is 10 m.
The acceleration satisfies
$|u_{i,x}|,|u_{i,y}|\leq4\,\mathrm{m/s^2}$.
Let $r_i^{\rm task}\in\Real^2$ be UAV $i$'s fixed assigned goal and
$e_i=(r_i^{\rm task}-p_i)/\|r_i^{\rm task}-p_i\|_2$ whenever $p_i\neq r_i^{\rm task}$.
The common task controller uses
$u_i^{\rm task}=\operatorname{sat}_{4\,\mathrm{m/s^2}}(((18\,\mathrm{m/s})e_i-v_i)/T_s)$ and returns zero at the goal,
where $\operatorname{sat}_{a}(z)=z\min\{1,a/\|z\|_2\}$ is Euclidean saturation with limit $a$ and $\operatorname{sat}_{a}(0)=0$.
We use $\lambda_1=\lambda_2=1\,\mathrm{s^{-1}}$ and $T_s=0.5$ s.
The desired predicted margin is $\tau=14\,\mathrm{m/s^2}$.

The fixed-reference positions are $p_{b_0}=(250,1000)$ m and $p_{b_1}=(250,1510)$ m in an $1800\times2000$ m world.
Let $D$ denote the vessel's offshore horizontal offset in meters.
With $e_{\rm off}=(1,0)$, $e_y=(0,1)$, and $p_{\rm mid}=(p_{b_0}+p_{b_1})/2$, the fixed task references are
\begin{align*}
r_1^{\rm task}&=p_{\rm mid}+\tfrac14D e_{\rm off}
                  +\tfrac{\sqrt3}{4}D e_y,\\
r_2^{\rm task}&=p_{\rm mid}+\tfrac12D e_{\rm off},\\
r_3^{\rm task}&=p_{\rm mid}+\tfrac34D e_{\rm off}
                  +\tfrac{\sqrt3}{4}D e_y,\\
r_4^{\rm task}&=p_{\rm vessel}=p_{\rm mid}+D e_{\rm off}.
\end{align*}
The critical vessel offset is $D_C=1218.862$ m, for which the longest terminal assigned edge is 840 m and therefore retains 10 m relative to the physical communication limit.
The low-pressure offset is $D_L=0.9D_C=1096.976$ m.
Both terminal geometries satisfy the static distance limits; the critical case increases relative speed and second-order braking demand during approach.

Table~\ref{tab:nominal-states} lists the nominal initial states.
\begin{table}[t]
\centering
\caption{Nominal initial states; $v_i[0]=(0,0)\,\mathrm{m/s}$.}
\label{tab:nominal-states}
\begin{tabular}{@{}ccc@{}}
\toprule
UAV & $p_{i,x}[0]$ (m) & $p_{i,y}[0]$ (m) \\
\midrule
UAV1 & 685.460911 & 1244.880280 \\
UAV2 & 681.101518 & 734.898912 \\
UAV3 & 1124.938034 & 986.114251 \\
UAV4 & 1120.578641 & 476.132883 \\
\bottomrule
\end{tabular}
\end{table}

Proposed and Baseline share the task reference, task controller, local barrier constraints, input bounds, HOCBF gains, sample time, and hard HOCBF-QP.
Baseline applies $u_i^0$ from \eqref{eq:baseline-qp}; Proposed applies Algorithm~\ref{alg:feedback} before the same hard HOCBF-QP.
We evaluate 20 matched initial conditions obtained by componentwise perturbations
$\Delta p_i\in[-2,2]^2$ m and $\Delta v_i\in[-0.2,0.2]^2$ m/s of each nominal UAV state.
Each pair uses the same initial state and task parameters for both controllers.
We use 20 fixed deterministic perturbations within these bounds, indexed by seeds 20262701--20262720 and provided with the public code.
The initial held accelerations are $u_j[-1]=0$.
Every perturbed initial state satisfies $h\geq0$ and $\dot h+\lambda_1h\geq0$ for all prescribed constraints.
We use Gurobi 12.0 with $10^{-9}$ feasibility and optimality tolerances.
We report hard infeasibility only when the QP status and reconstructed margin agree.

Let $m_{i,k}\in\{0,\ldots,8\}$ be UAV $i$'s selected candidate index at sample $k$, let $\alpha_{i,k}=m_{i,k}/8$, and let $u_{i,k}^{\rm sel}$ and $u_{i,k}^0$ be the selected nominal candidate and Baseline command at that sample.
We count one intervention whenever $\alpha_{i,k}>0$.
The associated task deviation compares these time-indexed commands at the same state and preceding neighbor accelerations:
\begin{equation}
 T_s\sum_{k,i:\,\alpha_{i,k}>0}
 \|u_{i,k}^{\rm sel}-u_{i,k}^0\|_2.
 \label{eq:task-deviation}
\end{equation}
\subsection{Multi-UAV Results}

Figures~\ref{fig:spatial} and~\ref{fig:paired-mechanism} show the representative event sequence and paired critical-task mechanism, while Table~\ref{tab:results} summarizes the task outcomes.

\begin{figure*}[!t]
\centering
\includegraphics[width=0.86\textwidth]{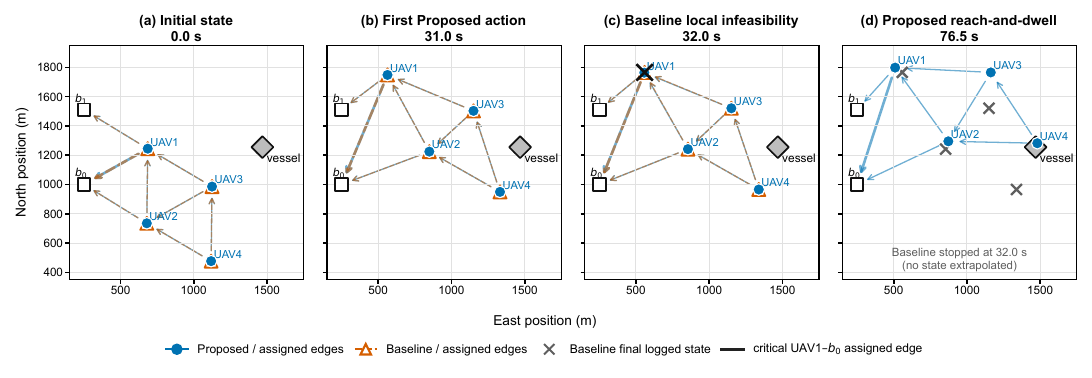}
\caption{Representative paired trajectory at four events. The final panel retains only Baseline's last available state.}
\label{fig:spatial}
\end{figure*}

\begin{figure*}[!t]
\centering
\includegraphics[width=0.84\textwidth]{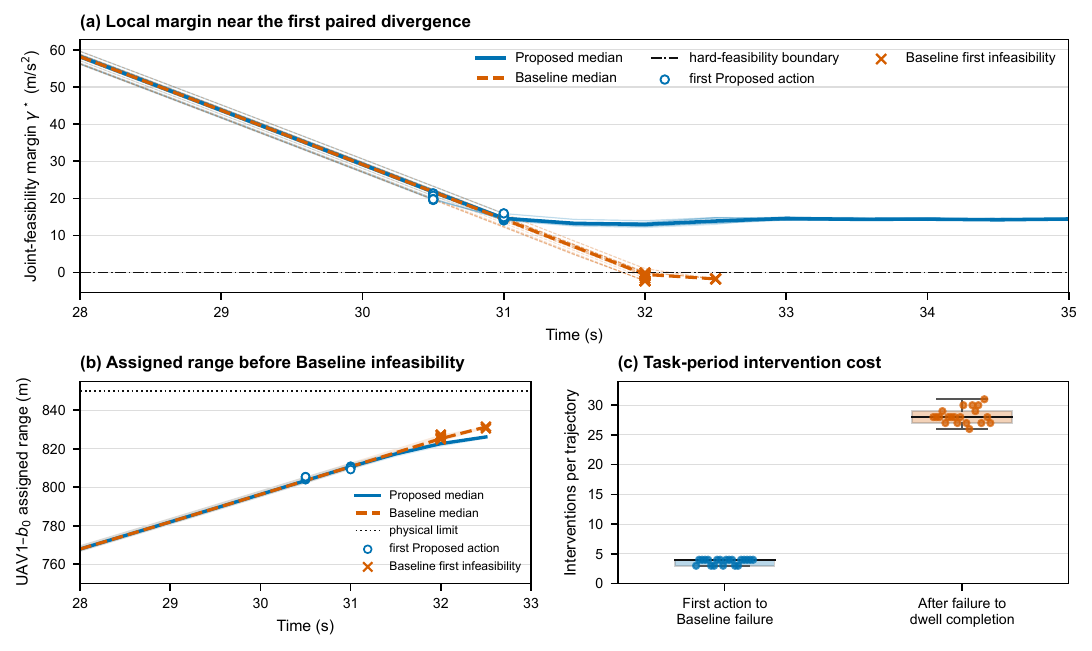}
\caption{Critical-task results over 20 pairs: median UAV1 local margin, assigned UAV1--$b_0$ range, and Proposed interventions before dwell completion.}
\label{fig:paired-mechanism}
\end{figure*}

Both metrics end at Proposed's first dwell completion; over that interval, the reported range and separation are extrema of the assigned communication edges and pairwise UAV distances across the 20 Proposed trajectories.

At $D_L$, both controllers complete all 20 tasks in 74.0--74.5 s, and Proposed never changes the Baseline command.
At $D_C$, Proposed completes all 20 tasks at 76.5 s, whereas every paired Baseline reaches local HOCBF-QP infeasibility before completion.

\begin{table}[t]
\centering
\caption{Results over 20 paired initial conditions.}
\label{tab:results}
\begin{tabular}{@{}lcc@{}}
\toprule
Metric & Low task & Critical task \\
\midrule
Proposed completion & 20/20 & 20/20 \\
Baseline completion & 20/20 & 0/20 \\
Proposed completion time (s) & 74.0--74.5 & 76.5 \\
Proposed interventions & 0 & 30--35 \\
Task deviation (m/s) & 0 & 40.94--45.08 \\
Maximum assigned range (m) & 822.63 & 840.06 \\
Minimum UAV separation (m) & 385.23 & 399.83 \\
\bottomrule
\end{tabular}
\end{table}

In the critical task, Proposed first changes UAV1's command at 30.5 s in eight pairs and 31.0 s in twelve.
UAV1's current margin at that instant is 14.072--21.339 $\mathrm{m/s^2}$.
Baseline first becomes locally infeasible at 32.0 s in 14 pairs and 32.5 s in six.
Its reconstructed UAV1 margin is $-2.212$ to $-0.013\,\mathrm{m/s^2}$, while paired Proposed has 12.128--14.913 $\mathrm{m/s^2}$ at the same time.
Ten first actions use the least-deviation branch; the other ten use maximum-margin fallback.
Across all 638 interventions, including the 74 preceding the paired Baseline failure, the one-step margin error is at most $1.14\times10^{-6}\,\mathrm{m/s^2}$ and never changes its classification relative to either $0$ or $\tau$.

Figure~\ref{fig:paired-mechanism}(a) shows the median UAV1 local margin across the 20 pairs and the individual event markers; Proposed acts before the paired Baseline margin becomes negative.
At Baseline's first infeasibility, its UAV1--$b_0$ assigned range is 824.793--831.503 m, below the 850 m physical limit; thus the local HOCBF conflict appears before a communication-range violation.
Proposed limits the growth of that edge and remains within 840.060 m over the completed task, as shown in Fig.~\ref{fig:paired-mechanism}(b).
This edge illustrates the dominant conflict in these trajectories but does not by itself define the multi-constraint margin.

Proposed uses 30--35 task-period interventions per critical trajectory.
Three or four occur between its first action and the paired Baseline infeasibility; 26--31 occur afterward but before dwell completion.
The corresponding task deviation \eqref{eq:task-deviation} is 40.942--45.082 m/s per trajectory.
The first-action window supports the initial paired divergence, whereas the later actions quantify the cost of continued regulation and task completion.
Across all completed Proposed trajectories, the physical communication and separation limits and the sampled hard HOCBF-QP constraints remain satisfied.

The evaluation covers 20 deterministic pairs with planar dynamics, a known stationary target, and fixed references.
The study does not vary $D_C$ or $\tau$ beyond their specified values; disturbances and real-sea operation remain for future evaluation.

\section{Conclusion}
\label{sec:conclusion}

This paper introduced a local joint-feasibility margin for bounded-input second-order barrier constraints and used its one-step prediction to modify a task command before an unchanged hard HOCBF-QP is solved.
A controlled one-UAV case isolated the predicted joint conflict and showed that the margin-feedback controller retained positive next-step common control authority.
In separate matched multi-UAV simulations, Proposed retains nominal behavior at low pressure and acts at positive margin in the critical task before Baseline reaches local infeasibility, then remains feasible at the same time and completes reach-and-dwell.
Together, these results support the margin-feedback controller as a one-step preventive layer for bounded-input second-order barrier controllers whose common control authority can shrink before current infeasibility.

\section*{ACKNOWLEDGMENT}
Generative AI tools assisted with manuscript restructuring and language editing; the authors verified the scientific content and conclusions.

\bibliographystyle{IEEEtran}
\bibliography{references}

\end{document}